\documentclass{article}

\PassOptionsToPackage{round}{natbib}
\usepackage[preprint]{neurips_2021}
\usepackage{natbib}
\usepackage[colorlinks=true, citecolor=darkblue, linkcolor=maroon]{hyperref} 
\usepackage{url}         
\usepackage{booktabs}     
\usepackage{amsfonts}      
\usepackage{nicefrac}           
\usepackage{xcolor}         
\usepackage{bm}
\usepackage{algorithm}
\usepackage[noend]{algorithmic}
\usepackage{amsmath}
\usepackage{amsthm}
\usepackage{amssymb}
\usepackage{mathtools}
\usepackage{mathrsfs}
\usepackage{graphicx}
\usepackage[skip=4pt]{subcaption}
\usepackage{subcaption}
\usepackage{enumitem}
\usepackage{multirow}
\usepackage{wrapfig}
\usepackage{pifont}
\usepackage{tablefootnote}
\usepackage{tikz}
\usepackage[most]{tcolorbox}

\usepackage{cleveref}

\crefformat{section}{\S#2#1#3}
\crefformat{subsection}{\S#2#1#3}
\crefformat{subsubsection}{\S#2#1#3}
\crefrangeformat{section}{\S\S#3#1#4 to~#5#2#6}
\crefmultiformat{section}{\S\S#2#1#3}{ and~#2#1#3}{, #2#1#3}{ and~#2#1#3}

\definecolor{darkblue}{RGB}{0,90,180}
\definecolor{maroon}{RGB}{255,0,255}
\definecolor{white}{RGB}{255,255,255}
\definecolor{blue1}{RGB}{230,245,255}
\definecolor{blue2}{RGB}{55,176,250}
\definecolor{mcmc}{RGB}{231, 231, 231}

\hypersetup{colorlinks,breaklinks,
            urlcolor=darkblue,
            linkcolor=maroon}

\newtheorem{definition}{Definition}
\newtheorem{assumption}{Assumption}
\newtheorem{thm}{Theorem}
\newtheorem{lemma}{Lemma}

\newtheorem{proposition}{Proposition}

\newcommand{\norm}[1]{\left\lVert #1 \right\rVert}
\DeclarePairedDelimiter{\abs}{\lvert}{\rvert}

\DeclareMathOperator{\argmin}{argmin}

\DeclareMathOperator{\Tr}{Tr}

\title{Safe Linear Leveling Bandits}

\author{%
  Ilker Demirel$\textsuperscript{*}$~~~~~~~~Mehmet Ufuk Ozdemir$\textsuperscript{*}$~~~~~~~~Cem Tekin \\
  \{ilkerd@ee, ufuk.ozdemir@ug, cemtekin@ee\}.bilkent.edu.tr \\
  Department of Electrical and Electronics Engineering \\
  Bilkent University, Ankara, Turkey \\
  $\textsuperscript{*}$Authors contributed equally
}

\begin{document}

\maketitle
\begin{abstract}
Multi-armed bandits (MAB) are extensively studied in various settings where the objective is to \textit{maximize} the actions' outcomes (i.e., rewards) over time. Since safety is crucial in many real-world problems, safe versions of MAB algorithms have also garnered considerable interest. In this work, we tackle a different critical task through the lens of \textit{linear stochastic bandits}, where the aim is to keep the actions' outcomes close to a target level while respecting a \textit{two-sided} safety constraint, which we call \textit{leveling}. Such a task is prevalent in numerous domains. Many healthcare problems, for instance, require keeping a physiological variable in a range and preferably close to a target level. The radical change in our objective necessitates a new acquisition strategy, which is at the heart of a MAB algorithm. We propose SALE-LTS: Safe Leveling via Linear Thompson Sampling algorithm, with a novel acquisition strategy to accommodate our task and show that it achieves sublinear regret with the same time and dimension dependence as previous works on the classical reward maximization problem absent any safety constraint. We demonstrate and discuss our algorithm's empirical performance in detail via thorough experiments.
\end{abstract}
\section{Introduction} \label{sec:intro}
The multi-armed bandit (MAB) is a reinforcement learning model where the trade-off between exploration and exploitation can be elegantly analyzed \cite{sutton1998introduction}. The MAB agent interacts with its environment through a set of arms (i.e.,\ actions). The agent plays an arm at each round and observes a noisy outcome called the instantaneous reward.  MAB is extensively studied under different model assumptions, where the aim is to maximize the actions' outcomes over time (i.e.,\ cumulative reward). Two prominent lines of algorithms have been proposed to tackle the exploration-exploitation dilemma in the MAB problem: upper-confidence bound (UCB) based and Thompson sampling (TS) based. UCB-based algorithms leverage the \textit{optimism in the face of uncertainty} (OFU) principle to form optimistic estimates of arm outcomes and play the arm that looks best under these estimates \cite{auer2002finite, abbasi2011improved, bubeck2012regret}. TS-based algorithms maintain a posterior distribution over the unknown parameters and rely on sampling random estimates to induce exploration \cite{thompson1933likelihood, russo2014learning, agrawal2012analysis}. \citet{abeille2017linear} proposes a TS-based algorithm for the linear stochastic bandit problem that leverages the OFU principle by randomizing the parameter estimates via an \textit{optimistic enough} distribution.

As many real-world problems demand strict safety guarantees during exploration, the development of safe exploration algorithms has received significant attention. A motivating example is clinical trials where the aim is to identify the best therapy while ensuring patient safety \cite{petroni2017implementation, villar2015multi}. \citet{amani2019linear} studies a safe variant of the UCB1 algorithm in \citet{auer2002finite} (safe linear UCB) preceded by a pure exploration phase. \citet{khezeli2020safe} proposes a safe algorithm for the linear bandit setting, which employs a pure exploration phase at the beginning followed by a greedy exploitation stage. \citet{moradipari2021safe} adapts the linear TS algorithm in \citet{abeille2017linear} to a safety constrained setting via some adjustments on the randomization procedure to leverage the OFU principle in the presence of safety constraints. \citet{usmanova2019safe} studies a safe variant of the Frank-Wolfe algorithm to solve a smooth optimization problem with unknown linear safety constraints and provides convergence guarantees but does not derive an upper bound on the regret.  \citet{sui2015safe,sui2018stagewise} consider safety for a general class of \textit{smooth} reward functions via Gaussian processes but do not provide formal upper bounds on the cumulative regret. Moreover, they model the safe set expansion as a proxy objective which leads to unnecessary suboptimal evaluations at the safe set boundaries. \citet{turchetta2019safe} proposes ``GoOSE'', a plug-in safe exploration algorithm, which consecutively queries new actions and observes the outcomes until the originally selected action is classified as safe or unsafe. This strategy is not applicable in a contextual setup since such ``re-observations'' may not be possible for a given context.

All of the works above consider the classical MAB setting where the objective is to maximize the actions' outcomes and consider a one-sided safety constraint. In this work, we study a different problem where the objective is not to maximize the actions' outcomes but to keep them close to a target level, which we call \textit{leveling}, while respecting a \textit{two-sided} safety constraint at each round. This problem structure is relevant and essential in many domains, particularly healthcare. One such example is blood pressure disorders, where the deviation from the standard value causes hypo-/hyper-tension events. Patient characteristics (i.e.,\ context) play a critical role in determining the blood pressure response to the therapeutic agent, and they should be taken into account during the dosing process \cite{materson1993single, nerenberg2018hypertension}. We formalize this problem from a MAB perspective via a suitable definition of the \textit{regret} as a proxy performance metric to capture our objective in \cref{sec:probform}. We consider a linear contextual stochastic bandit setting where the outcomes are a linear function of the actions \cite{agrawal2013thompson, dani2008stochastic, rusmevichientong2010linearly}. A closely related work using a modification of the LinUCB in \citet{chu2011contextual}, which does not consider any safety constraint, is \textit{``target tracking''} where the aim is to achieve closest outcomes to a moving target \cite{bregere2019target}. Our contributions are as follows.
\begin{tcolorbox}[colback=mcmc,colframe=mcmc]
\textbf{Key Contributions}
\begin{itemize}[leftmargin=*]
\item We introduce the \textit{leveling bandit} problem, along with a natural two-sided safety constraint in \cref{sec:probform}, which is relevant numerous domains and healthcare in particular. 
\item We propose SALE-LTS: Safe Leveling via Linear Thompson Sampling algorithm, which employs a modified acquisition strategy to accommodate the problem structure accompanied by a safe exploration scheme in \cref{sec:algo}.
\item We show that SALE-LTS incurs a cumulative regret of order $\mathcal{O} (T^{1/2} \log^{3/2}T \cdot d^{3/2} \log^{1/2} d)$, matching the LTS algorithm in \citet{abeille2017linear} while respecting the safety constraint in \cref{sec:regret}.
\end{itemize} 
\end{tcolorbox}
Finally, we make comprehensive \textit{in silico} experiments on type 1 diabetes mellitus (T1DM) disease in \cref{sec:exp}. T1DM patients regularly administrate \textit{bolus insulin} doses before meals since they can not regulate their blood glucose levels due to the pancreatic $\beta$-cell loss, and the adverse effects may lead to hospitalization and death \cite{bastaki2005diabetes}. We safely optimize the \textit{individualized} insulin dose recommendation models and compare our performance against the standard dose calculators.
\section{Problem Formulation} \label{sec:probform}
\paragraph{Notation}
We denote by $\diamond$ the vector appending operation, i.e., $a \diamond b = [a_1, \ldots, b_1, \ldots]$. $\norm{x}_2$ denotes the Euclidean norm of vector $x$, and $\lVert x \rVert_V$ denotes its weighted $\ell_2$-norm with respect to the positive semi-definite matrix $V$, i.e., $\lVert x \rVert_V = \sqrt{x^T V x}$. $[T]$ denotes the set $\{1,\ldots,T\}$, $\mathbb{I}$ is the indicator function such that $\mathbb{I} \{ A \} = 1$ when $A$ holds and 0 otherwise, and $\neg$ denotes the set complement operator. Also, throughout the paper Loewner order is used to compare two $n$ dimensional positive semi-definite matrices and denoted as $A \leq B$ if $ B - A$ is positive semi-definite. That is, $A \leq B$ implies that $x^T A x \leq x^T B x, \, \forall x \in \mathbb{R}^n$.  

Let $z_t \in \mathcal{Z}$ denote the context, and $a_t \in \mathcal{A}$ the action taken by the \textit{learner} upon observing $z_t$ at round $t \in [T]$. Both $\mathcal{Z} \subseteq \mathbb{R}^{d_z}$ and $\mathcal{A} \subseteq \mathbb{R}^{d_a}$ are compact and convex where $d_z$ and $d_a$ denote the dimensions of the corresponding spaces. Together, $z_t$ and $a_t$ form the \textit{pseudo-action} $x_t \in \mathcal{X}$, where $x_t \coloneqq z_t \diamond a_t$, $\mathcal{X} \coloneqq\mathcal{Z} \times \mathcal{A}$, and we denote the cardinality of $x \in \mathcal{X}$ by $d \coloneqq d_z + d_a$. For later convention, let us also denote by $\mathcal{X}_t \coloneqq \{ z_t \diamond a: a \in \mathcal{A} \}$.
\paragraph{Observations} At each round $t \in [T]$, the learner observes the context $z_t$ and plays an action $a_t$ accordingly resulting in the pseudo action $x_t = z_t \diamond a_t$. Then, it observes the following outcome upon taking the pseudo-action $x_t$,
\begin{align*} 
o_t = x_t^T \theta_* + \xi_t,
\end{align*} 
where $\theta_* \in \mathbb{R}^{d}$ is the fixed but \textit{unknown} parameter, and $\xi_t$ are zero-mean additive noise for $t \in [T]$. We define the filtration $\mathcal{F}_t = (\mathcal{F}_1, \sigma(x_1,\ldots,x_t,\xi_1,\ldots,\xi_t))$ representing the information gained up to round $t$.
\paragraph{Safety and Regret} The objective of the learner is to keep the outcomes close to a target level, $K$. Meanwhile, she must respect a two-sided safety constraint, $C_1 \leq x_t^T \theta_* \leq C_2$, where $C_1 \leq K \leq C_2$ are constants known to the learner. We formalize this objective as a MAB problem as follows,
\begin{align}
\text{minimize}~~~&R_T = \sum_{t=1}^T \abs{x_t^T \theta_* - K}, \label{eqn:regret} \\
\text{subject to}~~~&C_1 \leq x_t^T \theta_* \leq C_2,~~~\forall t \in [T], \label{eqn:safety_constraint}
\end{align}
where $R_T$ denotes the cumulative regret incurred by round $T$. We denote by $\mathcal{X}^s_t (\theta_*) \coloneqq \{x \in \mathcal{X}_t : C_1 \leq x^T \theta_* \leq C_2\}$ the set of safe pseudo-actions at round $t$. Then the condition in \eqref{eqn:safety_constraint} is equivalent to $x_t \in \mathcal{X}^s_t (\theta_*)$. However, $\mathcal{X}^s_t(\theta_*)$ is unknown to the learner as $\theta_*$ is. Therefore, the learner needs an additional machinery to guarantee (at least with high probability) the safety of the actions it takes. We assume that at every round $t$, an initial seed set $\mathcal{X}^{iss}_t (\theta_*) \subseteq \mathcal{X}^s_t (\theta_*)$ that is known to be safe is available to the learner. Existence of an initial safe seed set is a common assumption in safety-critical settings, and we need it to ensure the safety of the pseudo-actions at the earlier rounds. We denote $\mathcal{X}^{iss}_t (\theta_*)$ by $\mathcal{X}^{iss}_t$. We also make the following assumptions.

\begin{assumption} \label{a:a1}
For all $t \in [T]$, $\xi_t$ are conditionally zero-mean, R-sub-Gaussian noise variables for some constant $R \geq 0$. That is, $\mathbb{E} [\xi_t\mid \mathcal{F}_{t-1}] = 0$, and $E[e^{\alpha \xi_t} \mid \mathcal{F}_{t-1}] \leq \exp(\frac{\alpha^2 R^2}{2})$, $\forall \alpha \in \mathbb{R}$.
\end{assumption}

\begin{assumption}  \label{a:a2}
There exists constants $S,L \in \mathbb{R}^+$ such that $\norm{\theta_*}_2 \leq S$, and $\norm{x}_2 \leq L~,~\forall x \in \mathcal{X}$. 
\end{assumption}

\begin{assumption}  \label{a:a3}
There exists $a \in \mathcal{A}$ for all $z \in \mathcal{Z}$ such that $x^T \theta = K$, if $\norm{\theta_*-\theta} \leq \epsilon$ for some $\epsilon > 0$, where $x = z \diamond a$. 
\end{assumption}
Assumption~\ref{a:a3} states that for any context $z \in \mathcal{Z}$, there exist an action $a \in \mathcal{A}$ achieving the target level $K$ with respect to some $\theta$ within an $\epsilon$-neighborhood of $\theta_*$. The problem is ill-posed without this assumption. Consider the case where the Assumption~\ref{a:a3} does not hold. Then, for some context $z \in \mathcal{Z}$, even the optimal policy can not find $a \in \mathcal{A}$ such that $x^T \theta_* = K$ where $x = z \diamond a$. In such a case, one can simply introduce an alternative definition of regret with respect to the optimal policy, preserving the bounds derived in this work. For brevity, we make our analyses under Assumption~\ref{a:a3}.

\paragraph{Key Challenges} The key challenge lies in proposing a suitable \textit{acquisition} strategy to select the actions at each round, which is at the heart of a MAB algorithm. UCB-based algorithms leverage the aforementioned OFU principle to form \textit{optimistic} estimates of the actions and select accordingly. Consider the classical outcome maximization setting, and let ($x_*, \theta_*$) pair denote the optimal action and the true environment parameter, and ($x_t, \theta_t$) denote the learner's action and parameter estimate at round $t$. OFU principle (as in Eq. 7 of \citet{abbasi2011improved}) simply leads to the following,
\begin{align*}
x_*^T \theta_* - x_t^T \theta_t \leq 0,
\end{align*}
with high probability. That is, the estimated \textit{values} of the learner's actions are greater than the optimal value. This inequality plays a key role in bounding the regret of the UCB-based algorithms in terms of the convergence of the parameter estimate $\theta_t$, which improves over time as the observations accumulate. The same idea is employed in different settings, such as the celebrated GP-UCB algorithm in \citet{srinivas2010gaussian} as well, but the spirit is the same. However, in our case, the learner's objective is not to maximize the outcomes but to keep them close to a target level, as can be confirmed from the definition of the regret in \eqref{eqn:regret}. Therefore, the classical idea of optimism and UCB-based algorithms are not directly applicable for our task. In our case, optimism implies that $x_*^T \theta_* = x_t^T \theta_t$. We formalize this idea later in \cref{sec:regret}. Even though the TS-based algorithms are also mostly designed for outcome maximization, they are more flexible to build acquisition strategies to serve different purposes since they form the estimated action outcomes based on random sampling. We also consider a two-sided safety constraint. As the caution to take safe actions restrains an algorithm from exploring, performance and safety are competing objectives. In the following sections, we propose a simple algorithm to accommodate the problem structure studied in this work and provide theoretical guarantees on its performance and safety. 

\section{SALE-LTS Algorithm} \label{sec:algo}
\begin{wrapfigure}{L}{0.52\textwidth}
    \vspace{-20pt}
 \begin{minipage}{0.52\textwidth}
\begin{algorithm}[H]
\caption{SALE-LTS algorithm}
\label{alg:alg1}
\textbf{Input}: $T$, $\lambda$, $\delta$, $K$, $C_1$, $C_2$ 
\begin{algorithmic}[1] %[1] enables line numbers
\STATE Set $\delta' = \delta/(4T)$
\STATE Play a random $x_1$ in $\mathcal{X}^{iss}_1$
\STATE Observe $o_1 = x_1^T \theta_* + \xi_1$
\FOR{$t=2,\ldots, T$} 
\STATE Set $V_t = \lambda I_d + \sum_{i=1}^{t-1} x_i x_i^T$
\STATE Observe the context $z_t$
\STATE Compute the RLS-estimate $\hat{\theta}_t$ as in \eqref{eqn:rls}
\STATE Sample $\eta_t \sim D^{TS}$
\STATE Set $\tilde{\theta}_t = \hat{\theta}_t + \beta_t(\delta') V_t^{-\frac{1}{2}} \eta_t$
\STATE Form the confidence region:
\STATE $\mathcal{C}_t (\delta') = \{ \theta \in \mathbb{R}^{d}: \lVert \theta - \hat{\theta}_t \rVert_{V_t} \leq \beta_t (\delta') \}$
\STATE Form the \textit{proxy} safe set:
\STATE $\hat{\mathcal{X}}^s_t \! = \! \{ x \! \in \! \mathcal{X}_t \! : \! C_1 \! \leq \! x^T \theta \! \leq \! C_2,\forall \theta \in \mathcal{C}_t (\delta') \}$
\STATE $x_t \in \argmin_{x \in \hat{\mathcal{X}}^s_t \cup \mathcal{X}^{iss}_t } \abs{x^T \tilde{\theta}_t - K}$
\STATE Observe the outcome $o_t = x_t^T \theta_* + \xi_t$
\ENDFOR
\end{algorithmic}
\end{algorithm}
  \end{minipage}
\end{wrapfigure}

We propose SALE-LTS: Safe Leveling via Linear Thompson Sampling algorithm, a safe variant of the LTS algorithm in \citet{abeille2017linear} bearing two critical differences regarding the acquisition strategy and an additional safety mechanism. At round $t$, the learner computes the regularized least-squares (RLS) estimate $\hat{\theta}_t$ of the true parameter $\theta_*$ as,
\begin{align} \label{eqn:rls}
\hat{\theta}_t = V_t^{-1} \sum\nolimits_{i=1}^{t-1} o_i x_i,
\end{align}
where $V_t = \lambda I_d + \sum_{i=1}^{t-1} x_i x_i^T$ is the design matrix and $\lambda \in \mathbb{R}^+$ is the regularization parameter. Then, it forms a perturbed parameter $\tilde{\theta}_t$ to induce sufficient exploration as,
\begin{align} 
\tilde{\theta}_t = \hat{\theta}_t + \beta_t (\delta') V_t^{-\frac{1}{2}} \eta_t, \label{eqn:tildetheta}
\end{align}
where $\eta_t$ is sampled from a appropriate multivariate distribution $D^{TS}$.
\paragraph{Safety Mechanism} We employ a mechanism to ensure that the safety constraint (i.e., $C_1 \leq x_t^T \theta_* \leq C_2$) is well-attended. At a given round $t$, the learner forms a confidence region around the RLS estimate $\hat{\theta}_t$ that contains the true parameter ($\theta_*$) with high probability as,
\begin{align*}
\mathcal{C}_t (\delta') \coloneqq \{ \theta \in \mathbb{R}^{d} : \lVert \theta - \hat{\theta}_t \rVert_{V_t} \leq \beta_t(\delta') \},
\end{align*}
where $\beta_t(\delta)$ is as defined in Proposition~\ref{p:p1}. Henceforward, we drop the index and denote $\mathcal{C}_t (\delta')$ by $\mathcal{C}_t$, and $\beta_t(\delta')$ by $\beta_t$.
\begin{proposition} \label{p:p1}
(Theorem~2 in \citet{abbasi2011improved}). For any $\delta \in (0,1)$, $\lambda \in \mathbb{R}^+$, let $\beta_t(\delta) = R\sqrt{d \log(\frac{1 + \frac{(t-1) L^2}{\lambda}}{\delta})} + \sqrt{\lambda} S$. Under Assumptions \ref{a:a1} and \ref{a:a2}, $\theta_* \in \mathcal{C}_t (\delta)$ for all $t \geq 1$ with at least $1-\delta$ probability.
\end{proposition}
The learner is then able to form the following \textit{proxy} set of safe pseudo-actions as,
\begin{align} \label{eqn:proxysafe}
\hat{\mathcal{X}}^s_t = \{ x \in \mathcal{X}_t : C_1 \leq x^T \theta \leq C_2,~\forall \theta \in \mathcal{C}_t \}.
\end{align}
$\hat{\mathcal{X}}^s_t$ contains the pseudo-actions that satisfy the safety constraint for all $\theta \in \mathcal{C}_t$. Since $\theta_* \in \mathcal{C}_t$ with at least $1-\delta'$ probability, $x \in \hat{\mathcal{X}}^s_t$ are also safe with at least $1-\delta'$ probability. To facilitate later convenience, we provide an equivalent definition of the proxy safe set in \eqref{eqn:proxysafe} as,
\begin{align}
\hat{\mathcal{X}}^s_t &= \{x \in \mathcal{X}_t : \min_{\theta \in \mathcal{C}_t} x^T \theta \geq C_1,~\max_{\theta \in \mathcal{C}_t} x^T \theta \leq C_2 \} \label{eqn:proxysafealt0} \\
&= \{ x \in \mathcal{X}_t : x^T \hat{\theta}_t - \beta_t \norm{x}_{V_t^{-1}} \geq C_1, \nonumber \\
&\hspace{51pt}x^T \hat{\theta}_t + \beta_t \norm{x}_{V_t^{-1}} \leq C_2 \}. \label{eqn:proxysafealt}
\end{align}
Derivation of \eqref{eqn:proxysafealt} is provided in \cref{sec:appa}.
\paragraph{Acquisition Strategy} The \textit{leveling} objective necessitates a change in the acquisition strategy. The learner does not pick the action that maximizes the estimated outcome but the one that yields the closest outcome to the target level $K$, which marks the most distinctive aspect of our algorithm. The learner plays a pseudo-action at each round $t$ as,
\begin{align}
x_t \in \argmin_{x \in \hat{\mathcal{X}}^s_t  \cup \mathcal{X}^{iss}_t} \abs{x^T \tilde{\theta}_t - K}. \label{eqn:acqustrat}
\end{align}
\section{Regret Analyses}\label{sec:regret}

\subsection{Main Events}\label{sec:mainevents}
We define the following events,
\begin{align*}
\hat{E}_t &\coloneqq \{ \lVert \hat{\theta}_t - \theta_* \rVert_{V_t} \leq \beta_t (\delta')\} \\
\tilde{E}_t &\coloneqq \{ \lVert \tilde{\theta}_t - \hat{\theta}_t \rVert_{V_t} \leq \gamma_t (\delta')\} \\
E_t &\coloneqq \hat{E}_t \cap \tilde{E}_t, 
\end{align*}
where, $\gamma_t (\delta') \coloneqq \beta_t (\delta') \sqrt{c d \log (\frac{c'd}{\delta'})}$, $c,~c' \in \mathbb{R}^+$. Let us also define the following parameter set,
\begin{align*}
\mathcal{E}_t (\delta') \coloneqq \{ \theta \in \mathbb{R}^{d} : \lVert \theta - \hat{\theta}_t \rVert_{V_t} \leq \gamma_t(\delta') \},
\end{align*}
so that the event $\tilde{E}_t$ implies that $\tilde{\theta}_t \in \mathcal{E}_t (\delta')$. We denote $\gamma_t (\delta')$ by $\gamma_t$, and $\mathcal{E}_t(\delta')$ by $\mathcal{E}_t$. The events $\hat{E}_t$ and $\tilde{E}_t$ indicate that the RLS-estimate $\hat{\theta}_t$ concentrates around the true \textit{unknown} parameter $\theta_*$, and the \textit{perturbed} parameter $\tilde{\theta}_t$ concentrates around the RLS-estimate $\hat{\theta}_t$, respectively.
\subsection{Distribution for $\eta_t$}
Let us denote the optimal pseudo-action set with respect to $\tilde{\theta}_t$ at round $t$ by $X^*(\tilde{\theta}_t) \subset \mathcal{X}_t$. Formally, $X^*(\tilde{\theta}_t) \coloneqq \argmin_{x \in \mathcal{X}} \left( | x^T \tilde{\theta}_t - K | \right) $. We define the following \textit{optimistic} parameter set and event,
\begin{align*} 
\tilde{\Theta}^{opt}_t &\coloneqq \{\theta \in \mathbb{R}^{d} :\argmin_{x \in \hat{\mathcal{X}}^s_t} \abs{x^T \theta - K} \subset X^*(\tilde{\theta}_t) \} \\
D_t &\coloneqq \Big\{ \max_{x \in \mathcal{X}} \|x \|_{V_t^{-1}} < \frac{G}{2(\beta_T + \gamma_T)} \Big\},
\end{align*}
where $G = \min \{(C_2 - C_1)/2, \epsilon L\}$. $\tilde{\theta}_t \in \tilde{\Theta}^{opt}_t$ implies that $\argmin_{x \in \hat{\mathcal{X}}^s_t} \abs{x^T \tilde{\theta}_t - K} \subset X^*(\tilde{\theta}_t)$. That is, the safe set $\hat{\mathcal{X}}^s_t$ includes at least one optimal pseudo-action with respect to $\tilde{\theta}_t$. In the next definition, we formalize the distributional properties of $D^{TS}$ that we later leverage in the analyses.
\begin{definition} \label{def:def1}
Let $D^{TS}$ be a multivariate distribution on $\mathbb{R}^{d}$ absolutely continuous with respect to the Lebesgue measure with the following properties ($\eta_t \sim D^{TS}$):

1. There exists a strictly positive probability $p$ such that the following holds for any $\tilde{x} \in X^*(\tilde{\theta}_t)$ and any $u \in \mathbb{R}^{d}$ with $\lVert u \rVert_{2} = 1$,
\begin{align} \label{eqn:ntdef1} 
\mathbb{P} \left\lbrace \frac{C_1 - K}{\beta_t \lVert\tilde{x} \rVert_{V_t^{-1}}} + 1 \leq u^T \eta_t \leq \frac{C_2 - K}{\beta_t \lVert \tilde{x} \rVert_{V_t^{-1}}} - 1 \: \: \middle| D_t, E_t \right\rbrace > p.
\end{align}

2. There exists $c,c'~\mathbb{R}^+$ such that for any $\delta \in (0,1)$, the following \textit{concentration} property holds,
\begin{align} \label{eqn:ntdef2}
\mathbb{P} \{ \lVert \eta_t \rVert_2 \leq \sqrt{c d \log(\frac{c'd}{\delta} ) } \} \geq 1 - \delta.
\end{align}
\end{definition}
\eqref{eqn:ntdef1} implies that there exists a strictly positive probability $p$ such that $\mathbb{P} \{ \tilde{\theta}_t \in \tilde{\Theta}^{opt}_t \cap \mathcal{E}_t \mid D_t, E_t\} > p$, that is, the pseudo-action $x_t$ is \textit{optimistic}, and we have $x_t^T \tilde{\theta}_t = K$ which plays a central role in bounding the regret. The concentration inequality in \eqref{eqn:ntdef2} implies that $\tilde{\theta}_t$ is not too far from $\hat{\theta}_t$ so that the growth of the regret can be controlled. A zero-mean multivariate Gaussian distribution satisfies the conditions in Definition~\ref{def:def1}. Detailed derivations can be found in \cref{sec:appc}.

\begin{thm}\label{thm:thm1}
Let $G \coloneqq \min \{\frac{C_2 - C_1}{2}, \epsilon L\}$, $C \coloneqq \frac{C_2-C_1}{2}$, and pick $\lambda \geq \max\{1, L^2\}$. For any $\delta \in (0,1)$, under Assumptions \ref{a:a1} and \ref{a:a2}, the cumulative regref of SALE-LTS algorithm is upper bounded with at least $1-\delta$ probability as,
\begin{align*}
R_T &\leq \big( \beta_T  + (1 + \frac{2}{p}) \gamma_T \big) \sqrt{2 T d \log(1 + \frac{T L^2}{d \lambda})} \nonumber \\
& + \frac{2 \gamma_T}{p} \sqrt{\frac{8 T L^2}{\lambda} \log \frac{4}{\delta}} + \frac{6 C d}{\log \big( 1 + (\frac{G} {2(\beta_T + \gamma_T)})^2 \big) } \log \left( 1+ \frac{ L^2}{  \lambda \log \big( 1 + (\frac{G} {2(\beta_T + \gamma_T)})^2 \big) } \right),
\end{align*}
\end{thm}
which is of order $\mathcal{O} (T^{1/2} \log^{3/2}T \cdot d^{3/2} \log^{1/2} d)$, matching the LTS algorithm in \citet{abeille2017linear} for the outcome maximization setting absent any safety constraint. Note that the dependence of $\beta_t (\delta')$ and $\gamma_T (\delta')$ on $\delta'$ is suppressed for brevity and they are denoted as $\beta_t$ and $\gamma_t$. The definifitions of $\beta_t$, $\gamma_t$ and $p$ are given in Proposition~\ref{p:p1}, \cref{sec:mainevents} and Definition~\ref{def:def1} respectively.

\subsection{Preliminary Results} \label{sec:regbound}
Before bounding the regret, we state some useful preliminary results. In the next lemma, we show that the number of rounds where $D_t$ does not hold is bounded.
\begin{lemma} \label{l:badevent_paper}
Let $\tau \subset [T]$ denote the set of rounds where $\mathbb{I} \{ \neg D_t \} = 1$ for $t \in \tau$. We have,
\begin{align*}
 \abs{\tau} \leq  \frac{3d}{\log \big( 1 + (\frac{G} {2(\beta_T + \gamma_T)})^2 \big) } \log \left( 1+ \frac{ L^2}{  \lambda \log \big( 1 + (\frac{G} {2(\beta_T + \gamma_T)})^2 \big) } \right).
\end{align*}
\end{lemma} 
Detailed proof for Lemma~\ref{l:badevent_paper} is provided in \cref{sec:appb}.
\begin{proposition} \label{p:p2}
Let $\hat{E} \coloneqq \hat{E}_T \cap \dots \cap \hat{E}_1$, and $\delta' = \delta/(4T)$. Then $\mathbb{P} \{ \hat{E} \} \geq 1- \delta/4$.
\end{proposition}
The proof follows after a simple union bound over all rounds $t \leq T$ using the result in Proposition~\ref{p:p1} similar to the proof of Lemma 1 in \citet{abeille2017linear}, and omitted for conciseness.
\begin{proposition} \label{p:p3}
Let $\tilde{E} \coloneqq \tilde{E}_T \cap \dots \cap \tilde{E}_1$, and $\delta' = \delta/(4T)$. Then $\mathbb{P} \{ \tilde{E} \} \geq 1- \delta/4$.
\end{proposition}
\begin{proof}
\begin{align*}
\mathbb{P} &\big\{ \lVert \tilde{\theta}_t - \hat{\theta}_t  \rVert_{V_t} \leq \gamma_t \big\} = \mathbb{P} \big\{ \lVert \beta_t V_t^{-\frac{1}{2}} \eta_t \rVert_{V_t} \leq \gamma_t \big\} \\ 
&= \mathbb{P} \{ \lVert \eta_t \rVert_{2} \leq \frac{\gamma_t }{\beta_t} \} = \mathbb{P} \{ \lVert \eta_t \rVert_{2} \leq \sqrt{c d \log (\frac{c'd}{\delta'})} \} > 1 - \delta',
\end{align*}
where the first equality is by \eqref{eqn:tildetheta}, and the last from \eqref{eqn:ntdef2}. The proof follows after a union bound over all $t \leq T$. 
\end{proof}
Finally, let $E \coloneqq \hat{E} \cap \tilde{E}$. Then, we have $\mathbb{P} \{E\} \geq 1 - \delta/2$ by Propositions \ref{p:p2} and \ref{p:p3}, since $\mathbb{P} \{\hat{E} \cap \tilde{E}\} \geq 1 - \mathbb{P} \{\neg \hat{E}\} - \mathbb{P} \{ \neg \tilde{E} \} > 1 - \delta/2$. 

\begin{proposition}\label{p:p4}
(\citet{abbasi2011improved} Lemma 11) Let $t \leq T$, $\lambda \geq \max \{1,L^2\}$. Given a sequence of pseudo-actions $(x_1, \ldots, x_t) \in \mathcal{X}^t$, and the corresponding Gram matrix $V_t$, we have,
\begin{align*} 
\sum_{i=1}^T \norm{x_i}_{V_i^{-1}} &\leq\Big(T \sum_{i=1}^T \norm{x_i}^2_{V_i^{-1}} \Big)^{\frac{1}{2}} \leq \sqrt{2Td \log \big( 1 + \frac{T L^2}{d \lambda} \big)}.
\end{align*}
\end{proposition}
Finally, we recall the Azuma's concentration inequality for super-martingales.
\begin{proposition} \label{p:p5}
If a super-martingale $(Y_t)_{t \geq 0}$ corresponding to a filtration $\mathcal{F}_t$ satisfies $\abs{Y_t - Y_{t-1}} < c_t$ for some constant $c_t$ for all $t=1,\ldots,T$, then for any $\alpha > 0 $,
\begin{align*}
\mathbb{P} \{ Y_T - Y_0 \geq \alpha \} \leq 2 e^{-\frac{\alpha^2}{2 \sum_{t=1}^T c_t^2}}.
\end{align*}
\end{proposition}
\subsection{Bounding the Regret}
We assume that the event $E$ holds throughout the analysis and consider the following decomposition of the instantaneous regret in \eqref{eqn:regret},
\begin{align*}
r_t &\coloneqq \abs{x_t^T \theta_* - K} = \abs{x_t^T \theta_* - x_t^T \tilde{\theta}_t + x_t^T \tilde{\theta}_t - K} \nonumber \\
&\leq \underbrace{\abs{x_t^T \theta_* - x_t^T \tilde{\theta}_t}}_{\text{Term 1},~r_{t,1}} +  \underbrace{\abs{x_t^T \tilde{\theta}_t - K}}_{\text{Term 2},~r_{t,2}},
\end{align*}
\paragraph{Bounding Term 1}
Consider the following for $r_{t,1}$,
\begin{align}
r_{t,1} &= \abs{x_t^T \theta_* - x_t^T \hat{\theta}_t + x_t^T \hat{\theta}_t - x_t^T \tilde{\theta}_t} \nonumber \\
&\leq \abs{x_t^T \theta_* - x_t^T \hat{\theta}_t } + \abs{ x_t^T \hat{\theta}_t - x_t^T \tilde{\theta}_t} \nonumber \\
&\leq \| x_t \|_{V_t^{-1}} \| \theta_* - \hat{\theta}_t \|_{V_t} + \| x_t \|_{V_t^{-1}} \| \hat{\theta_t} - \tilde{\theta}_t \|_{V_t}  \label{eqn:bt1_2} \\
&\leq \| x_t \|_{V_t^{-1}} \beta_T + \| x_t \|_{V_t^{-1}} \gamma_T = (\beta_T + \gamma_T) \| x_t \|_{V_t^{-1}}, \label{eqn:bt1_3}
\end{align}
where \eqref{eqn:bt1_2} holds by Cauchy-Schwarz inequality, and \eqref{eqn:bt1_3} follows since the event $E_t$ holds and $(\beta_t)_t$ and $(\gamma_t)_t$ are increasing sequences. Then, by Proposition~\ref{p:p4}, we have,
\begin{align*}
R_{T,1} \leq (\beta_T + \gamma_T) \sqrt{2Td \log \big( 1 + \frac{T L^2}{d \lambda} \big)}. 
\end{align*}
\paragraph{Bounding Term 2}
Consider the following for $r_{t,2}$,
\begin{align*}
R_{T,2} = \underbrace{\sum_{t=1}^T r_{t,2} \mathbb{I} \{ \neg D_t, E_t \}}_{R_{T,2}^1} + \underbrace{\sum_{t=1}^T r_{t,2} \mathbb{I} \{D_t, E_t \}}_{R_{T,2}^2}.
\end{align*}
$R_{T,2}^1 \leq \frac{6 C d}{\log \big( 1 + (\frac{G} {2(\beta_T + \gamma_T)})^2 \big) } \log \left( 1+ \frac{ L^2}{  \lambda \log \big( 1 + (\frac{G} {2(\beta_T + \gamma_T)})^2 \big) } \right)$ by Lemma~\ref{l:badevent_paper}, and since the worst case regret is upper-bounded by $2 C = C_2 - C_1$ with high probability ($1-\delta$) by \eqref{eqn:proxysafe} and \eqref{eqn:acqustrat}. Even in the worst case our algorithm plays safe actions (with $1-\delta$ probability), i.e. the output is between $C_2$ and $C_1$ in every round. Therefore, the regret incurred in any round is upper bounded by $C_2-C_1$. We can then proceed to bound $R_{T,2}^2$.
Let $X_t(\theta) \coloneqq \argmin_{x \in \hat{\mathcal{X}}^s_t \cup \mathcal{X}^{iss}_t} \abs{x^T\theta - K}$ denote the set of optimal \textit{safe} pseudo-actions at round $t$ with respect to $\theta$ ($X_t (\theta) \subset \mathcal{X}_t$), and $x_t(\theta) \in X_t(\theta)$.
\begin{align*}
r_{t,2} &= \abs{x_t^T \tilde{\theta}_t - K} \leq \sup_{\theta \in \mathcal{E}_t} \abs{x_t(\theta)^T \theta -K},
\end{align*}
since $\tilde{\theta}_t \in \mathcal{E}_t$ when the event $\tilde{E}_t$ holds. We can also bound $r_{t,2}$ by the expectation of the right-hand side over any random choice of $\tilde{\theta}_t \in \tilde{\Theta}^{opt}_t \cap \mathcal{E}_t$ as (let $\mathcal{S}_t \coloneqq \tilde{\Theta}^{opt}_t \cap \mathcal{E}_t$),
\begin{align}
&r_{t,2} \mathbb{I} \{D_t, E_t \} \nonumber \\
&\leq \mathbb{E} \left[ \sup_{\theta \in \mathcal{E}_t} \abs{x_t(\theta)^T \theta - K} \mathbb{I} \{D_t, E_t \}  \: \: \middle| \mathcal{F}_t, \tilde{\theta}_t \in \mathcal{S}_t \right] \nonumber \\
&= \mathbb{E} \left[ \sup_{\theta \in \mathcal{E}_t} \abs{x_t(\theta)^T \theta - x_t ^T \tilde{\theta}_t} \: \: \middle| \mathcal{F}_t, \tilde{\theta}_t \in \mathcal{S}_t, D_t, E_t \right] \mathbb{P} \{D_t, E_t \} \label{eqn:bt2_4} \\
&\leq \mathbb{E} \left[ \sup_{\theta \in \mathcal{E}_t} \abs{x_t^T \theta - x_t ^T \tilde{\theta}_t} \: \: \middle| \mathcal{F}_t, \tilde{\theta}_t \in \mathcal{S}_t, D_t, E_t \right] \mathbb{P} \{D_t, E_t \} \label{eqn:bt2_3} \\
&= \mathbb{E} \left[ \sup_{\theta \in \mathcal{E}_t} \abs{x_t^T ( \tilde{\theta}_t - \theta)} \: \: \middle| \mathcal{F}_t, \tilde{\theta}_t \in \mathcal{S}_t, D_t, E_t \right] \mathbb{P} \{D_t, E_t \} \nonumber \\
&\leq \mathbb{E} \left[ \lVert x_t \rVert_{V_t^{-1}} \sup_{\theta \in \mathcal{E}_t} \lVert \tilde{\theta}_t - \theta \rVert_{V_t} \: \: \middle| \mathcal{F}_t, \tilde{\theta}_t \in \mathcal{S}_t, D_t, E_t  \right] \mathbb{P} \{D_t, E_t \} \label{eqn:bt2_5} \\
&\leq \mathbb{E} \left[  \lVert x_t \rVert_{V_t^{-1}} \left( \lVert \tilde{\theta}_t - \hat{\theta}_t \rVert_{V_t} + \sup_{\theta \in \mathcal{E}_t} \lVert \hat{\theta}_t - \theta \rVert_{V_t} \right) \: \: \middle| \mathcal{F}_t, \tilde{\theta}_t \in \mathcal{S}_t, D_t, E_t  \right] \mathbb{P} \{D_t, E_t \} \label{eqn:bt2_6} \\
&\leq \mathbb{E} \left[  \lVert x_t \rVert_{V_t^{-1}} 2 \gamma_t \: \: \middle| \mathcal{F}_t, \tilde{\theta}_t \in \mathcal{S}_t, D_t, E_t \right] \mathbb{P} \{D_t, E_t \}, \label{eqn:bt2_7}
\end{align}
where \eqref{eqn:bt2_4} holds since we have $x_t^T \tilde{\theta}_t = K$ given $\tilde{\theta}_t \in \tilde{\Theta}^{opt}_t$, and \eqref{eqn:bt2_3} follows from the definition of $X_t(\theta)$ and $x_t(\theta) \in X_t(\theta)$. \eqref{eqn:bt2_5} and \eqref{eqn:bt2_6} follow from Cauchy-Schwarz and triangle inequality, respectively, and \eqref{eqn:bt2_7} follows since the event $\tilde{E}_t$ holds.

Then, since we have $\mathbb{P} \{ \tilde{\theta}_t \in\mathcal{S}_t | D_t, E_t\} > p$,
\begin{align}
r_{t,2} \mathbb{I} \{D_t, E_t \} &p \leq \mathbb{E} \left[ 2 \gamma_t \lVert x_t \rVert_{V_t^{-1}}  \: \: \middle| \mathcal{F}_t, \tilde{\theta}_t \in \mathcal{S}_t, D_t, E_t \right] \nonumber \\
& \hspace{10pt} \times \mathbb{P} \{ \tilde{\theta}_t \in \mathcal{S}_t | D_t, E_t \} \mathbb{P} \{D_t, E_t\} \nonumber \\
&\leq 2 \gamma_t \mathbb{E} \left[ \lVert x_t \rVert_{V_t^{-1}} \middle| \mathcal{F}_t, D_t, E_t \right] \mathbb{P} \{D_t, E_t \} \label{eqn:btp2_1}   \\
&=2 \gamma_t \mathbb{E} \left[ \lVert x_t \rVert_{V_t^{-1}} \mathbb{I} \{ D_t, E_t \} \middle| \mathcal{F}_t \right], \nonumber
\end{align}
where \eqref{eqn:btp2_1} holds since $\lVert x_t \rVert_{V_t^{-1}}$ is non-negative. Then we have,
\begin{align*}
R_{T,2}^2 &\leq \frac{2 \gamma_T}{p} \bigg( \underbrace{\sum_{t=1}^T \| x_t \|_{V_t^{-1}}}_{R_{T,2}^{2,1}} + \underbrace{\sum_{t=1}^T \Big(\mathbb{E} \left[ \lVert x_t \rVert_{V_t^{-1}} \mid \mathcal{F}_t \right] - \| x_t \|_{V_t^{-1}}\Big)}_{R_{T,2}^{2,2}}  \bigg),
\end{align*}
where $R_{T,2}^{2,1} \leq \sqrt{2 T d \log(1+\frac{T L^2}{d \lambda})}$ by Proposition~\ref{p:p4}. Since $(R_{t,2}^{2,2})_{t=1}^T$ is a martingale by construction, we can bound it by Azuma's inequality in Proposition~\ref{p:p5}. Since $\| x_t\|_{2} \leq L$, and $V_t^{-1} \leq \frac{1}{\lambda} I$, we have $\mathbb{E} \left[ \lVert x_t \rVert_{V_t^{-1}} \mid \mathcal{F}_t \right] - \| x_t \|_{V_t^{-1}} \leq \frac{2 L}{\sqrt{\lambda}}$ for all $t \geq 1$. Then, we can bound $R_{T,2}^{2,2}$ with at least $1-\delta/2$ probability as,
\begin{align*} 
R_{T,2}^{2,2} \leq \sqrt{\frac{8 T L^2}{\lambda} \log \frac{4}{\delta}}.
\end{align*}
Both event $E$ and Azuma's inequality hold with at least $1- \delta/2$ probability. Finally, the upper bound on the cumulative regret is obtained after observing that $R_T = R_{T,1} + R_{T,2}^1 + R_{T,2}^2$.
\section{Experiments} \label{sec:exp}
\subsection{The Simulator} \label{ssec:thesim}
Clinical experimentation with real patients is risky, and it runs into ethical challenges \cite{chen2020ethical, vayena2018machine, price2017regulating}. Dose-finding studies for clinical trials either test their methods through synthetic simulations or use exogenous algorithms to learn a dose-response model as the ground truth from real-world data when the patient group is homogeneous \cite{aziz2021multi, shen2020learning, lee2020contextual}. The latter strategy is not applicable in our case since we aim to learn individualized dose-response models. We conduct \textit{in~silico} experiments using the University of Virginia (UVa)/PADOVA T1DM simulator \cite{kovatchev2009silico, simglucose2018}. It simulates the patients' postprandial blood glucose (PPBG) levels for given meal events (i.e., carbohydrate intake and fasting blood glucose level) and bolus insulin dose pairs via a complex model using differential equations and patient characteristics. It is recognized by the United States Food and Drug Administration (U.S.\ FDA) as a reliable hormone controller design framework, and is the most frequently used framework in blood glucose control studies \cite{daskalaki2013personalized, zhu2020basal, zhu2020insulin, tejedor2020reinforcement}. We experiment with all  30 virtual patients that come with the simulator: 10 children, 10 adolescents, and 10 adults. We also compare the SALE-LTS' performance against clinicians' for five virtual patients to provide external validation and evaluate SALE-LTS' potential as a supplementary tool in clinical settings.
\begin{table*}[h]
\caption{The ``First Round'' row presents the performance metrics averaged over first dose recommendations for each (patient, meal event) pair. The ``Overall'' row presents the same metrics averaged over all dose recommendations. The ``Clinician Experiment'' row presents the performance metrics averaged over all dose recommendations for 5 patients and 20 (test) meal events in the experiment with the clinicians. The target PPBG level is 112.5~mg/dl ($K$), and the lower and the upper values for safe BG levels are 70~mg/dl ($C_1$) and 180~mg/dl ($C_2$), respectively.}
\label{table:results}
\centering
\small
\setlength\tabcolsep{3pt}
\begin{tabular}{clccccc@{\hspace{.2cm}}c@{\hspace{.2cm}}c}
\hline \addlinespace[0.02cm] &
\textbf{Algorithm}  & \textbf{PPBG} & \textbf{Safe Freq.} & \textbf{Hyper Freq.} & \textbf{Hypo Freq.} & \textbf{HBGI}  & \textbf{LBGI}  & \textbf{RI}  \\
\addlinespace[0.02cm]
\hline
\addlinespace[0.01cm]
\multicolumn{1}{c}{\multirow{3}{*}{\textbf{First Round}}} &
\multicolumn{1}{l}{Tuned Calc.} & 122.9$\pm$17.7 & 0.998 & 0.000 & 0.002 & 1.07 & 0.87 & 1.94 \\ \addlinespace[0.01cm] \multicolumn{1}{c}{} &
\multicolumn{1}{l}{LE-LTS} & 99.8$\pm$47.6 & 0.756 & 0.054 & 0.190 & 6.81 & 11.58 & 18.39 \\ \addlinespace[0.01cm] \multicolumn{1}{c}{} &
\multicolumn{1}{l}{SALE-LTS} & \textbf{120.75$\pm$15.31} &  \textbf{0.999} &  \textbf{0.000} &  \textbf{0.001} &  \textbf{0.75} &  \textbf{0.72} &  \textbf{1.47} \\ 
\addlinespace[0.01cm]
\hline
\addlinespace[0.02cm]
\multicolumn{1}{c}{\multirow{3}{*}{\textbf{Overall}}} &
\multicolumn{1}{l}{Tuned Calc.} & 123.0$\pm$17.5 & 0.99800 & \textbf{0.00000} & 0.00200 & 1.05 & 0.91 & 1.96 \\ \addlinespace[0.01cm] \multicolumn{1}{c}{} & 
\multicolumn{1}{l}{LE-LTS} & 110.05$\pm$19.8 & 0.96882 & 0.00859 & 0.02259 & 0.81 & 1.75 & 2.56 \\ \addlinespace[0.01cm] \multicolumn{1}{c}{} & 
\multicolumn{1}{l}{SALE-LTS} & \textbf{114.92$\pm$9.9}  & \textbf{0.99900} & 0.00007 & \textbf{0.00059} & \textbf{0.30} & \textbf{0.19} & \textbf{0.49} \\ 
\addlinespace[0.01cm]
\specialrule{.2em}{.1em}{.1em} 
\addlinespace[0.02cm]
\multicolumn{1}{c}{\multirow{3}{*}{\begin{tabular}[c]{@{}c@{}}\textbf{Clinician}\\ \textbf{Experiment}\end{tabular}}} &
\multicolumn{1}{l}{Calculator} & 158.2$\pm$27.5 & 0.85 & 0.15 & \textbf{0.00} & 4.684 & \textbf{0.001} & 4.685 \\ \addlinespace[0.01cm] \multicolumn{1}{c}{} & 
\multicolumn{1}{l}{Clinicians} & 176.4$\pm$34.6 & 0.59 & 0.41 & \textbf{0.00} & 7.876 & 0.002 & 7.878 \\ \addlinespace[0.01cm] \multicolumn{1}{c}{} & 
\multicolumn{1}{l}{SALE-LTS} & \textbf{146.8$\pm$25.7}  & \textbf{0.92} & \textbf{0.08} & \textbf{0.00} & \textbf{3.077} & 0.046 & \textbf{3.123} \\ 
\hline
\addlinespace[0.01cm]
\end{tabular}
\end{table*} 
\subsection{Performance Metrics} \label{ssec:permet}
Patients experience hypoglycemia (hypo) or hyperglycemia (hyper) events when their blood glucose (BG) levels fall below 70 mg/dl or go beyond 180 mg/dl, respectively. These events can have severe adverse effects and may lead to hospitalization or even death \cite{bastaki2005diabetes}. We set the target BG level to 112.5 mg/dl \cite{kovatchev2003algorithmic}. Our objective is to optimize the course of treatment for each patient by recommending doses that lead to PPBG levels close to the target BG level (i.e.,\ minimizing the regret) while avoiding the hypo and hyper events. A well-balanced algorithm should incur a low cumulative regret together with low hypo and hyper frequencies (error frequencies) and risk indices. Risk indices are defined as low blood glycemic index (LBGI), high blood glycemic index (HBGI), and risk index (RI $\coloneqq$ LBGI + HBGI), and they characterize the risk of hypo and hyper events in the long term \cite{kovatchev2003algorithmic}.
\subsection{Simulation Procedure} \label{ssec:simpro}
We learn individualized models for each patient. First, we create 30 different meal events. A meal event is a (carbohydrate intake, fasting blood glucose) pair (i.e., the context vector $z_t$). We sample carbohydrate intake for each meal event from $[20,~80]$~g, and fasting blood glucose from $[100,~150]$~mg/dl, both uniformly. We use the same 30 meal events for each patient. Next, we make consecutive bolus insulin dose recommendations (i.e., the action $a_t$) to patients for these meal events in a round-robin fashion. To be precise, after making a recommendation for a meal event, the model makes recommendations for the other 29 meal events, and it updates its parameters between each consecutive dose recommendation upon observing the outcomes (PPBG). We run a total of 15 rounds. That is, we make 15 different recommendations for each meal event and 450 (30 $\times$ 15) total for each patient. The motivation behind this procedure is two-fold: \textbf{(i)} it includes a variety of meal events similar to a patient's routine diet plan. \textbf{(ii)} it tests whether observations related to different meal events help improve the dose recommendations for other meal events (i.e.,\ inter-contextual information transfer) thanks to the round-robin strategy. We simulate both the original (SALE-LTS) and the \textit{unsafe} version (LE-LTS) of our algorithm to justify the safety mechanism's effectiveness further. Finally, we average the results from all 30 runs (one for each patient).
\subsection{Bolus Insulin Dose Calculators} \label{ssec:dosecalc}
Formula-based calculators are the \textit{de facto} tools for insulin dose recommendation in diabetes care \cite{walsh2011guidelines}. They make recommendations via simple calculations using carbohydrate intake, fasting blood glucose, target blood glucose, and patient characteristics such as carbohydrate factor (CRF) and correction factor (CF). Calculators must be fine-tuned (e.g., CRF and CF parameters must be accurate) for each patient to be safe and effective, which is a challenging task. Even when fine-tuned, they may discount some patient characteristics which affect the PPBG. \textit{Correction~doses} constitute 9\% of the patients' daily insulin dose intake due to the calculator's failure \cite{walsh2011guidelines}. We tune the calculator parameters for each patient to ensure it makes safe recommendations more than 99\% of the time overall. Then, we use the recommendations by the tuned calculator (tuned calc.) to initialize the safe dose set (i.e., $\mathcal{X}^{iss}_t$) for each (patient, meal event) pair. 
\subsection{Discussion of Results} \label{ssec:dor}
\paragraph{Safety} \label{sssec:safety} Assuring patient safety while optimizing the course of treatment is the most pivotal task. One could expect that to be particularly challenging at the beginning (i.e., exploration phase) since there is not much information to learn and infer from. Therefore, we probe the algorithms' safety statistics in the first round (i.e., first recommendations for each 30 meal events for each patient). Table~\ref{table:results} shows that SALE-LTS ensures patient safety almost perfectly, while LE-LTS (the \textit{unsafe} version) incurs a significantly higher error frequency of 0.244 (0.054 $+$ 0.190) in the first round. Moreover, we observe that LE-LTS yields alarmingly high risk index values of 6.81 and 11.58 for HBGI and LBGI, respectively. This observation suggests that LE-LTS not only makes recommendations that result outside of the safe region ((70,~180)~mg/dl) but also dangerously away from it. We confirm that presumption after viewing Figure~\ref{fig:boxplot_init} which shows that SALE-LTS behaves more cautiously in the beginning phase, and it does not significantly deviate from the recommendations provided by the tuned calculator. However, it still yields better error frequencies, risk index values, and PPBG distribution than the tuned calculator. That suggests the information gained after making a recommendation for a meal event helps improve the recommendations for other meal events.
\begin{figure*}[!t]
\vspace{-10pt}
    \centering
    \begin{minipage}[t]{.32\linewidth}
    \captionsetup{width=.9\linewidth}
        \centering
        \includegraphics[width=0.98\linewidth]{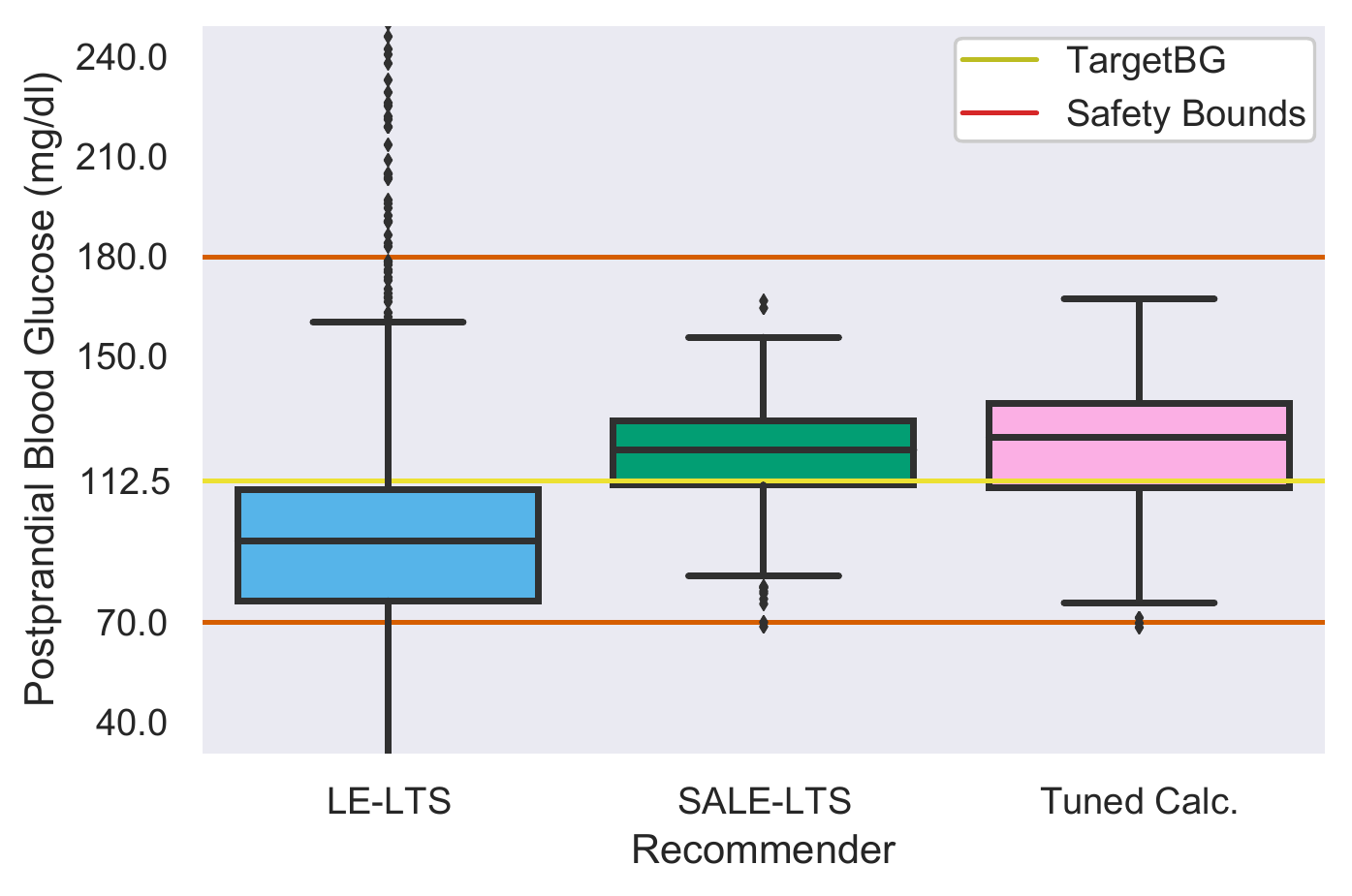}
        \caption{PPBG distributions over all patients after the \textit{first round} of recommendations is completed.}
        \label{fig:boxplot_init}
    \end{minipage}%
    \begin{minipage}[t]{.32\linewidth}
    \captionsetup{width=.9\linewidth}
        \centering
        \includegraphics[width=0.98\linewidth]{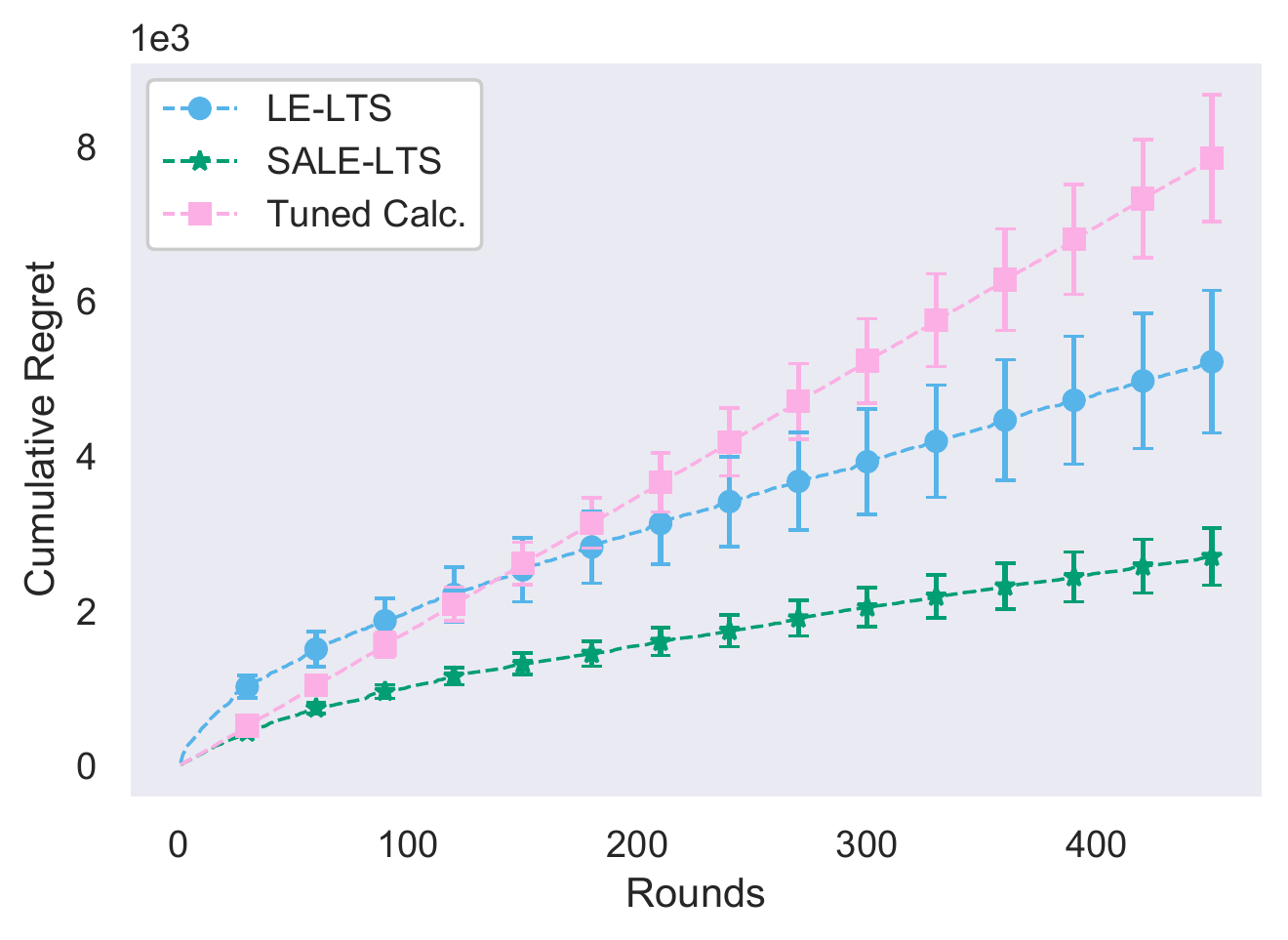}
        \caption{Cumulative regrets averaged over all recommendations made to all (patient, meal event) pairs.}
        \label{fig:regret}
    \end{minipage}%
    \begin{minipage}[t]{.32\linewidth}
    \captionsetup{width=.9\linewidth}
	\centering
	\includegraphics[width=0.98\linewidth]{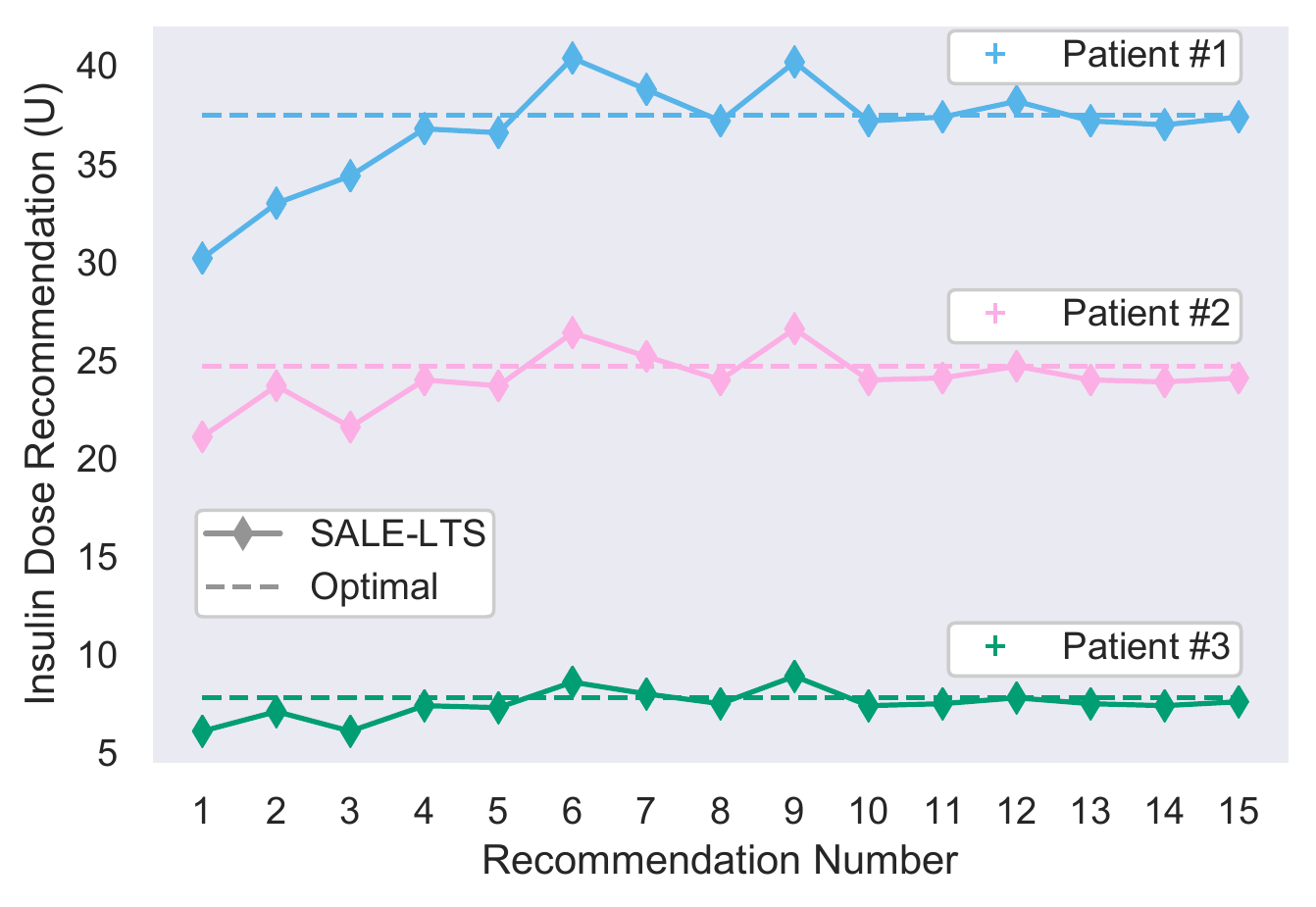}
	\caption{Consecutive and optimal recommendations for three different patients for the same meal event.}
	\label{fig:comparison}
\end{minipage}%
    \vspace{-15pt}
\end{figure*} 
\paragraph{Regret} \label{sssec:regret} The objective is to minimize regret while satisfying the safety constraints. By \eqref{eqn:regret}, this translates to recommending doses that lead to PPBG values close to the target BG level. Table~\ref{table:results} (the ``Overall'' row) and Figure~\ref{fig:regret} shows that both SALE-LTS and LE-LTS significantly outperforms the tuned calculator. Moreover, Figure~\ref{fig:regret} signifies another crucial advantage of the safety mechanism. At the initial exploration phase ($\approx$100 rounds), LE-LTS incurs the highest cumulative regret. Meanwhile, SALE-LTS takes advantage of having access to the extra information provided by the tuned calculator (i.e., the initial safe set, $\mathcal{X}^{iss}_t$) and yields the lowest cumulative regret.
\paragraph{Inter-Contextual Information Transfer}  \label{sssec:icit} A patient's routine diet plan naturally consists of various meals. Therefore, it is crucial to test an algorithm's ability to leverage the information gained from a meal event (i.e., context) for another one. We have briefly discussed SALE-LTS' ability to perform inter-contextual information transfer towards the end of the \textbf{Safety} section. Figure~\ref{fig:comparison} shows consecutive dose recommendations for a fixed meal event for three different patients. Remember that due to the round-robin strategy we use in our simulations, the algorithm makes recommendations for other 29 meal events between two consecutive recommendations for the same meal event. We observe that these recommendations for other meal events help SALE-LTS improve its recommendations for the fixed meal event and converge to the optimal dose recommendation for each patient.
\begin{wrapfigure}{R}{0.32\textwidth}
    \vspace{-15pt}
\begin{minipage}{\linewidth}
    \captionsetup{width=.9\linewidth}
	\centering
	\includegraphics[width=0.98\linewidth]{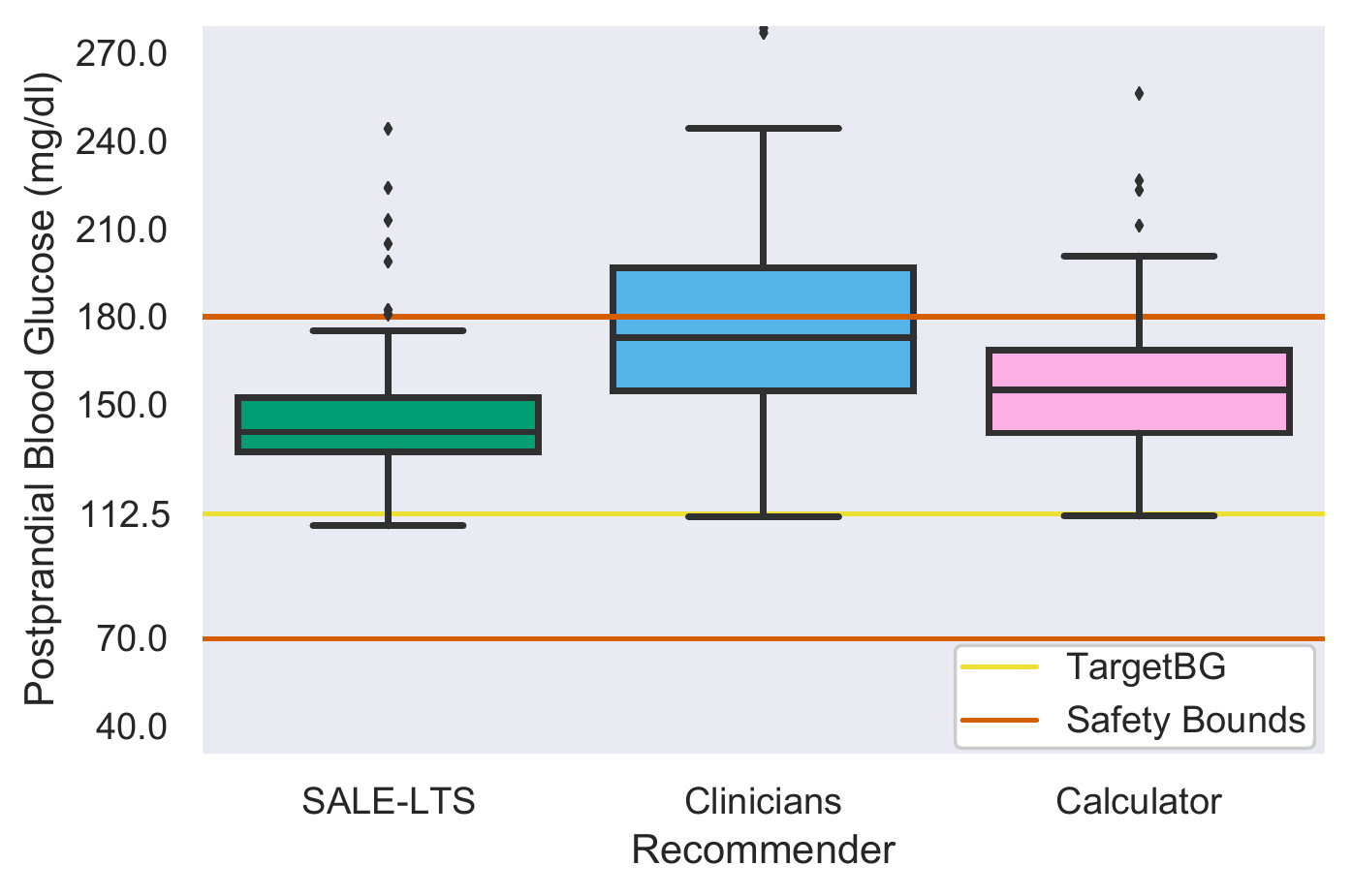}
	\caption{PPBG distributions in clinician experiment with five virtual patients.}
	\label{fig:doc_comp}
\end{minipage}%
    \vspace{-10pt}
\end{wrapfigure}
\paragraph{Comparison Against Clinicians}  \label{sssec:cac} In our best effort to provide external validation, we compare SALE-LTS' performance against clinicians' for five virtual patients. We do not use fine-tuned calculators to complicate the task further. For each patient, we provide clinicians with 20 training samples in the form of (meal event, insulin dose, PPBG measurement) and ask them to make dose recommendations for 20 new \textit{unseen} test scenarios. We provide SALE-LTS with the same 20 training samples and query recommendations for the same 20 test samples. We do not perform any updates to the model using the observations related to the test samples. Figure~\ref{fig:doc_comp} and Table~\ref{table:results} show that SALE-LTS significantly outperforms the calculator and the clinicians both in terms of the safety and the PPBG distributions (i.e., regret). Even though this is a small experiment conducted on virtual patients, it still showcases that making inferences about a patient's dose-response model is not trivial at all. SALE-LTS is demonstrably successful in optimizing the course of treatment while ensuring patient safety, and it is promising as a supplementary tool in clinical settings.
\section{Concluding Remarks} \label{sec:conc}
We introduced the \textit{leveling bandit} problem, which is somewhat overlooked yet relevant in various domains, particularly in healthcare. We considered a two-sided safety constraint that naturally emerged along with the leveling problem and proposed SALE-LTS, a TS-based algorithm for the linear contextual stochastic bandit problem. We derived an upper bound on its regret and showed that it matches the original LTS algorithm while respecting the safety constraint. We demonstrated SALE-LTS' effectiveness against the baselines via detailed \textit{in silico} experiments for bolus insulin dose recommendation problem in T1DM disease.
\section*{Acknowledgements}
This study was supported in part by the Scientific and Technological Research Council of Turkey under Grant 215E342. Ilker Demirel is supported by Vodafone within the framework of 5G and Beyond Joint Graduate Support Programme coordinated by Information and Communication Techonologies Authority.

\appendix
\section{Derivation of Equation \eqref{eqn:proxysafealt} from \eqref{eqn:proxysafe}} \label{sec:appa}
At each round $t \in [T]$, the confidence region around the RLS-estimate $\hat{\theta}_t$ is formed as,
\begin{align*}
\mathcal{C}_t \coloneqq \{ \theta \in \mathbb{R}^{d} : \lVert \theta - \hat{\theta}_t \rVert_{V_t} \leq \beta_t \},
\end{align*}
and the proxy set of safe pseudo-actions is then obtained as,
\begin{align*}
\hat{\mathcal{X}}^s_t &= \{ x \in \mathcal{X}_t : C_1 \leq x^T \theta \leq C_2,~\forall \theta \in \mathcal{C}_t \} \\
&=\{ x \in \mathcal{X}_t : C_1 \leq \min_{\theta \in \mathcal{C}_t} x^T \theta,~\max_{\theta \in \mathcal{C}_t} x^T \theta \leq C_2 \}. 
\end{align*}
Note that $\forall \theta \in \mathcal{C}_t$ and $\forall x \in \mathcal{X}_t$,
\begin{align}
\abs{x^T (\theta - \hat{\theta}_t)} &\leq \| x \|_{V_t^{-1}} \| \theta - \hat{\theta}_t \|_{V_t} \label{eqn:APA1} \\
&\leq \| x \|_{V_t^{-1}} \beta_t, \label{eqn:APA2}
\end{align}
where \eqref{eqn:APA1} follows from Cauchy-Schwarz inequality and \eqref{eqn:APA2} is due to the definition of $\mathcal{C}_t$. Then, $\forall  \theta \in \mathcal{C}_t$ and $\forall x \in \mathcal{X}_t$, we have,
\begin{align}\label{eqn:APA3}
x^T \hat{\theta}_t - \beta_t \| x \|_{V_t^{-1}} \leq x^T \theta \leq  x^T \hat{\theta}_t + \beta_t \| x \|_{V_t^{-1}}.
\end{align}
For any $\hat{x} \in \mathcal{X}_t$, there exists $\hat{\theta} \in \mathcal{C}_t$ such that $\hat{x} = \alpha (\hat{\theta} - \hat{\theta}_t)$ for some constant $\alpha$ since $ \mathcal{C}_t$ is an ellipsoid centered at  $\hat{\theta}_t$. In that case, Cauchy-Schwarz implies an equality and \eqref{eqn:APA2} is equivalent to,
\begin{align} \label{eqn:APA6}
\abs{\hat{x}^T (\hat{\theta} - \hat{\theta}_t)} = \| \hat{x} \|_{V_t^{-1}} \beta_t.
\end{align}
Then, for any $x \in \mathcal{X}_t$, there exists $\theta_1, \theta_2 \in \mathcal{C}_t$ such that,
\begin{align}
x^T \theta_1 &= x^T \hat{\theta}_t + \| x \|_{V_t^{-1}} \beta_t \label{eqn:APA4} \\
x^T \theta_2 &= x^T \hat{\theta}_t -\| x \|_{V_t^{-1}} \beta_t, \label{eqn:APA5}
\end{align}
that is, the lower and upper bounds in \eqref{eqn:APA3} for any $x \in \mathcal{X}_t$ is satisfied for some $\theta_1, \theta_2 \in \mathcal{C}_t$. We then have,
\begin{align*}
\min_{\theta \in \mathcal{C}_t} x^T\theta &= x^T \hat{\theta}_t - \beta_t \| x \|_{V_t^{-1}} \\
\max_{\theta \in \mathcal{C}_t} x^T\theta &= x^T \hat{\theta}_t + \beta_t \| x \|_{V_t^{-1}},
\end{align*} 
and we can insert $x^T \hat{\theta}_t - \beta_t \| x \|_{V_t^{-1}}$ and $x^T \hat{\theta}_t + \beta_t \| x \|_{V_t^{-1}}$ as $\min$ and $\max$ values in \eqref{eqn:proxysafealt0} to obtain \eqref{eqn:proxysafealt}.

\section{Proof of Lemma~\ref{l:badevent_paper}} \label{sec:appb}
Let us first remember the definition of the event $D_t$ and restate the Lemma~\ref{l:badevent_paper}.
\begin{align*}
D_t = \Big\{ \max_{x \in \mathcal{X}} \|x \|_{V_t^{-1}} < \frac{G}{2(\beta_T + \gamma_T)} \Big\},
\end{align*}
where $G = \min \{\frac{C_2 - C_1}{2}, \epsilon L\}$.
\paragraph{Lemma 1.} \textit{Let $\tau \subset [T]$ denote the set of rounds where $\mathbb{I} \{ \neg D_t \} = 1$ for $t \in \tau$. We have,}
\begin{align*}
 \abs{\tau} \leq  \frac{3d}{\log \big( 1 + (\frac{G} {2(\beta_T + \gamma_T)})^2 \big) } \log \left( 1+ \frac{ L^2}{  \lambda \log \big( 1 + (\frac{G} {2(\beta_T + \gamma_T)})^2 \big) } \right).
\end{align*}
\begin{proof}
Let $\tau' \subset [T]$ denote the set of rounds where $\|x_t \|_{V_t^{-1}} \geq \frac{G}{2(\beta_T + \gamma_T)}$ for $t \in \tau'$. Let us define $G_t = \lambda I_d + \sum_{s=1}^{t-1} \mathbb{I} \{ s \in \tau' \} x_s x_s^T$. Observe that,
\begin{align}
&\left(\frac{d\lambda + \abs{\tau'} L^2}{d} \right)^d = \left(\frac{d\lambda + \sum_{s=1}^{T} \mathbb{I} \{s \in \tau'\} L^2}{d} \right)^d \nonumber \\
&\hspace{10pt}\geq  \left(\frac{\Tr(\lambda I_d) + \sum_{s=1}^{T}\mathbb{I} \{s \in \tau'\} \, \Tr(x_s x_s^T)}{d} \right)^d \label{eqn:badevent1} \\ 
&\hspace{10pt}= \left(\frac{\Tr(G_{T+1})}{d} \right)^d \nonumber \\
&\hspace{10pt}\geq \det (G_{T+1}) \nonumber \nonumber \\
&\hspace{10pt}= \det (\lambda I_d) \prod_{s \in \tau'} \left( 1 + \| x_s \|_{G_{s}^{-1}}^2 \right) \nonumber \\ 
&\hspace{10pt}\geq \det (\lambda I_d) \prod_{s \in \tau'} \left( 1 + \| x_s \|_{V_{s}^{-1}}^2 \right) \label{eqn:badevent2} \\ 
&\hspace{10pt}\geq \lambda^d \left( 1 + \bigg(\frac{G} {2(\beta_T + \gamma_T)}\bigg)^2  \right)^{\abs{\tau'}}, \label{eqn:badevent3}
\end{align}
where \eqref{eqn:badevent1} follows from $\lVert x \rVert_2 \leq L,~\forall x \in \mathcal{X}$ (see~Assumption~\ref{a:a2}), \eqref{eqn:badevent2} holds since $\| x_s \|_{G_{s}^{-1}} \geq \| x_s \|_{V_{s}^{-1}}$ for any $x \in \mathcal{X}$, and \eqref{eqn:badevent3} holds since $\|x_s \|_{V_s^{-1}} \geq \frac{G}{2(\beta_T + \gamma_T)}$ for $s \in \tau'$. We then have,
\begin{align*}
    \left(\frac{d\lambda + \abs{\tau'} L^2}{d} \right)^d \geq \lambda^d \left( 1 + \bigg(\frac{G} {2(\beta_T + \gamma_T)} \bigg)^2  \right)^{\abs{\tau'}},
\end{align*}
which is equivalent to,
\begin{align} \label{eqn:badevent4}
\abs{\tau'} \leq \frac{d}{\log \big( 1 + ( \frac{G} {2(\beta_T + \gamma_T)} )^2  \big)} \log \left(1 + \frac{ \abs{\tau'} L^2}{\lambda d} \right) 
\end{align}
Let $a \coloneqq \frac{d}{\log \big( 1 + ( \frac{G} {2(\beta_T + \gamma_T)})^2\big) }$ and $b \coloneqq \frac{L^2}{d\lambda}$, where $a,b > 0$.  \eqref{eqn:badevent4} is then equivalent to,
\begin{align} \label{eqn:badevent5}
\abs{\tau'} \leq a \log \left(1 + b \abs{\tau'} \right)
\end{align}
We know that the following is true for any $a,b > 0$,
\begin{align*}
a \log \Big(1+ b\big(3a \log(1+ab)\big)\Big) &\leq a \log (1+3a^2b^2) \\
&\leq a \log (1+ab)^3 \\
&= 3a\log(1+ab),
\end{align*}
that is,
\begin{align}\label{eqn:badevent6}
3a\log(1+ab) \geq a \log \Big(1+ b\big(3a \log(1+ab)\big)\Big).
\end{align}
Notice the similarity between \eqref{eqn:badevent5} and \eqref{eqn:badevent6}. If we can show that $c - a \log(1+bc)$ is increasing for $c > 3a \log(1+ab)$, then we can conclude $\abs{\tau'} \leq 3a \log(1+ab)$ via \eqref{eqn:badevent5} and \eqref{eqn:badevent6}. Consider the following,
\begin{align*}
\frac{\partial}{\partial c} (c - a \log(1+bc) ) = 1 - \frac {ab}{1+bc}
\end{align*}
Then, $c - a \log(1+bc)$ is increasing when 
\begin{align*}
1- \frac {ab}{1+bc} > 0,~~\text{i.e.,}~~c > a - \frac{1}{b}.
\end{align*}
Since $3a \log(1+ab) >  a - \frac{1}{b}$, we are done. One can confirm that by letting $m \coloneqq ab$ ($m>0$, since $a,b>0$) and observing that $3m \log(1+m) - m  + 1 >  0$ for all $m > 0$. Then, we have the following after plugging in the original values for $a$ and $b$.
\begin{align} \label{eqn:badevent7}
 \abs{\tau'} \leq  \frac{3d}{\log \big( 1 + (\frac{G} {2(\beta_T + \gamma_T)}))^2 \big) } \log \left( 1+ \frac{ L^2}{  \lambda \log \big( 1 + (\frac{G} {2(\beta_T + \gamma_T)})^2 \big) } \right).
\end{align}
Finally, by \eqref{eqn:badevent7}, for any $t > \abs{\tau'}$ and any $x_t \in \mathcal{X}$, we have $\lVert x_t \rVert_{V_{t}^{-1}} < \frac{G}{2(\beta_T + \gamma_T)}$, that is, $\max_{x \in \mathcal{X}} \lVert x \rVert_{V_{t}^{-1}} < \frac{G}{2(\beta_T + \gamma_T)}$, which implies that $\mathbb{I} \{ D_t \}$ = 1 and $\abs{\tau} \leq \abs{\tau'}$, completing the proof.
\end{proof}

\section{Distribution for $\eta_t$} \label{sec:appc}
We characterize the distributional properties of $\eta_t \sim D^{TS}$ that proves useful in bounding the regret. Let us denote the optimal pseudo-action set with respect to $\tilde{\theta}_t$ at round $t$ by $X^*(\tilde{\theta}_t) \subset \mathcal{X}_t$, and recall the definitions of the \textit{optimistic} parameter set and the event $D_t$,
\begin{align} 
\tilde{\Theta}^{opt}_t &= \{\theta \in \mathbb{R}^{d} :\argmin_{x \in \hat{\mathcal{X}}^s_t} \abs{x^T \theta - K} \in X^*(\tilde{\theta}_t) \} \label{eqn:thetatilde_sup} \\
D_t &= \Big\{ \max_{x \in \mathcal{X}} \|x \|_{V_t^{-1}} < \frac{G}{2(\beta_T + \gamma_T)} \Big\}. \label{eqn:new_d_event_sup}
\end{align}
where $G=\min \{(C_2-C_1)/2, \epsilon L\}$. Remember the definition for $D^{TS}$,
\paragraph{Definition 1.}
Let $D^{TS}$ be a multivariate distribution on $\mathbb{R}^{d}$ absolutely continuous with respect to the Lebesgue measure with the following properties ($\eta_t \sim D^{TS}$):

1. There exists a strictly positive probability $p$ such that the following holds for any $\tilde{x} \in X^*(\tilde{\theta}_t)$ and any $u \in \mathbb{R}^{d}$ with $\lVert u \rVert_{2} = 1$,
\begin{align} \label{eqn:ntdef1_sup} 
\mathbb{P} \left\lbrace \frac{C_1 - K}{\beta_t \lVert\tilde{x} \rVert_{V_t^{-1}}} + 1 \leq u^T \eta_t \leq \frac{C_2 - K}{\beta_t \lVert \tilde{x} \rVert_{V_t^{-1}}} - 1 \: \: \middle| D_t, E_t \right\rbrace > p.
\end{align}

2. There exists $c,c'~\mathbb{R}^+$ such that for any $\delta \in (0,1)$, the following \textit{concentration} property holds,
\begin{align} \label{eqn:ntdef2_sup}
\mathbb{P} \{ \lVert \eta_t \rVert_2 \leq \sqrt{c d \log(\frac{c'd}{\delta} ) } \} > 1 - \delta.
\end{align}
The \textit{concentration} inequality in \eqref{eqn:ntdef2_sup} is satisfied by $\eta_t \sim \mathcal{N} (0, \sigma^2 I_d)$. Observe that,
\begin{align}\label{eqn:ntdef22_sup}
\mathbb{P} \{ \lVert \eta \rVert_2 \leq \alpha \sqrt{d} \} \geq 1 - d \mathbb{P} \{ \abs{\eta_i} > \alpha \} \geq 1 - d 2 e^{\frac{-\alpha^2}{2 \sigma^2}}.
\end{align}
Letting $\alpha = \sqrt{c \log(\frac{c'd}{\delta})}$, $c=2 \sigma^2,~c'=2$, we have $\alpha = \sqrt{2 \sigma^2 \log (\frac{2d}{\delta})}$. \eqref{eqn:ntdef22_sup} then transforms to $\mathbb{P} \{ \lVert \eta \rVert_2 \leq \sqrt{c d \log(\frac{c'd}{\delta} ) } \} \geq 1- \delta$.

We define the following proxy event for brevity,
\begin{align*}
A_t \coloneqq \left\lbrace \frac{C_1 - K}{\beta_t \lVert\tilde{x} \rVert_{V_t^{-1}}} + 1 \leq u^T \eta_t \leq \frac{C_2 - K}{\beta_t \lVert \tilde{x} \rVert_{V_t^{-1}}} - 1 \right\rbrace
\end{align*}

For \eqref{eqn:ntdef1_sup} to be valid, the length of the interval $\left[ \frac{C_1 - K}{\beta_t \lVert\tilde{x} \rVert_{V_t^{-1}}} + 1, \frac{C_2 - K}{\beta_t \lVert \tilde{x} \rVert_{V_t^{-1}}} - 1 \right]$ must be lower bounded by a constant greater than zero.  Consider the following,
\begin{align}
&2 < \left( \frac{C_2 - K}{\beta_t \lVert \tilde{x} \rVert_{V_t^{-1}}} - 1 \right) - \left( \frac{C_1 - K}{\beta_t \lVert \tilde{x} \rVert_{V_t^{-1}}} + 1 \right) \label{eqn:auxdd_0} \\
&\hspace{-15pt}\iff \frac{C_1 - K}{\beta_t \lVert \tilde{x} \rVert_{V_t^{-1}}} + 2 < \frac{C_2 - K}{\beta_t \lVert \tilde{x} \rVert_{V_t^{-1}}} - 2 \label{eqn:auxdd_0} \\
&\hspace{-15pt}\iff~~~ 4 \beta_t \lVert \tilde{x} \rVert_{V_t^{-1}} < C_2 - C_1. \label{eqn:auxdd_1}
\end{align}
\eqref{eqn:auxdd_1} holds under the event $D_t$ since,
\begin{align}
4 \beta_t \lVert \tilde{x} \rVert_{V_t^{-1}} &< 4 \beta_t \max_{x \in \mathcal{X}}\lVert x \rVert_{V_t^{-1}} \nonumber \\
&< 4 \beta_t \frac{G}{2(\beta_T + \gamma_T)} \label{eqn:auxdd_2} \\
&\leq 4 \beta_t \frac{C_2-C_1}{4(\beta_T + \gamma_T)} \label{eqn:auxdd_3} \\
&< C_2 - C_1 \label{eqn:auxdddd}
\end{align}
where we have \eqref{eqn:auxdd_2} when the event $D_t$ holds, \eqref{eqn:auxdd_3} since $G = \min \{(C_2-C_1)/2, \epsilon L\}$, and \eqref{eqn:auxdddd} since $(\beta_t)_t$ is an increasing sequence, and $\gamma_t > 0$ for all $t$.  Since $u^T \eta_t \sim \mathcal{N} (0, \sigma^2)$, we have $\mathbb{P} \{ A_t \mid D_t\} > 2p$ for some $2p \in (0,1)$. Notice that in \eqref{eqn:ntdef1_sup}, the $A_t$ event is also conditioned on the $E_t$ event together with $D_t$. We then have,
\begin{align}
\mathbb{P} \{A_t \mid D_t, E_t\} &= \frac{\mathbb{P} \{A_t, E_t \mid D_t \}}{\mathbb{P} \{ E_t \mid D_t\}} \nonumber \\
&> \mathbb{P} \{A_t, E_t \mid D_t \} \nonumber \\ 
&\geq \mathbb{P} \{A_t \mid D_t\} - \mathbb{P} \{ \neg E_t \mid D_t \} \nonumber \\
&=\mathbb{P} \{A_t \mid D_t\} - \mathbb{P} \{ \neg E_t \} \nonumber \\
&>2p - 2\delta' \label{eqn:auxdnew}
\end{align}
where $\delta' = \delta/(4T)$, and since $ \mathbb{P} \{ E_t \} =  \mathbb{P} \{ \hat{E}_t \cap \tilde{E}_t \} \geq 1-2\delta'$ . When $T>\delta/(2p)$, we have $\mathbb{P} \{A_t \mid D_t, E_t\} > p$ by \eqref{eqn:auxdnew}, which confirms \eqref{eqn:ntdef1_sup}. 

When the events $D_t$ and $E_t$ hold, we have the following for any $x \in \mathcal{X}$,
\begin{align}
\abs{x^T (\theta_* - \tilde{\theta}_t)} &\leq \lVert x \rVert_{V_t^{-1}} \lVert \theta_* - \tilde{\theta}_t \rVert_{V_t} \label{eqn:auxdd_6} \\
&\leq \lVert x \rVert_{V_t^{-1}} \big(\lVert \theta_* - \hat{\theta}_t \rVert_{V_t} +  \lVert \hat{\theta}_t - \tilde{\theta}_t \rVert_{V_t}\big) \nonumber \\
&\leq \frac{G}{2(\beta_T + \gamma_T)} (\beta_t + \gamma_t) \label{eqn:auxdd_7} \\
&\leq \frac{G}{\beta_T + \gamma_T} (\beta_t + \gamma_t)  \\
&\leq \frac{\epsilon L}{\beta_T + \gamma_T} (\beta_t + \gamma_t), \label{eqn:auxdd_8} \\
&\leq \epsilon L \label{eqn:auxdd_9}
\end{align}
where \eqref{eqn:auxdd_6} holds by Cauchy-Schwarz inequality, \eqref{eqn:auxdd_7} follows since the events $D_t$ and $E_t$ hold, \eqref{eqn:auxdd_8} holds since $G = \min \{(C_2-C_1)/2, \epsilon L\}$, and \eqref{eqn:auxdd_9} holds since since $(\beta_t)_t$ and $(\gamma_t)_t$ are increasing sequences.
Then, since we have $\lVert x \rVert_{2} \leq L$ by Assumption~\ref{a:a2}, we can state the following,
\begin{align}
&\max_{x \in \mathcal{X}} \abs{x^T (\theta_* - \tilde{\theta}_t)} = L \lVert \theta_* - \tilde{\theta}_t \rVert_{2} \leq \epsilon L \nonumber \\
&\hspace{-15pt}\implies~~~ \lVert \theta_* - \tilde{\theta}_t \rVert_{2} \leq \epsilon. \label{eqn:auxdd_10}
\end{align}
Then, by Assumption~\ref{a:a3} and \eqref{eqn:auxdd_10}, we can state that there exists $x \in \mathcal{X}_t$ such that $x^T \tilde{\theta}_t = K$. Then, for all $\tilde{x} \in X^*(\tilde{\theta}_t)$, we have $\tilde{x}^T \tilde{\theta}_t = K$.  Next, let $u^T = \frac{-\beta_t \tilde{x}^T V_t^{-\frac{1}{2}}}{\beta_t \norm{\tilde{x}}_{V_t^{-1}}}$. Then, by \eqref{eqn:ntdef1_sup}, we can state the following with at least $p$ probability,
\begin{align} 
K - \beta_t \tilde{x}^T V_t^{-\frac{1}{2}} \eta_t - \beta_t \lVert\tilde{x} \rVert_{V_t^{-1}} &\geq C_1 \label{eqn:auxdd_4} \\
K - \beta_t\tilde{x}^T V_t^{-\frac{1}{2}} \eta_t + \beta_t \lVert \tilde{x} \rVert_{V_t^{-1}} &\leq C_2. \label{eqn:auxdd_5}
\end{align}
We can then replace $K$ with $\tilde{x}^T \tilde{\theta}_t$ for any $\tilde{x} \in X^*(\tilde{\theta}_t)$, and $\tilde{\theta}_t$ with $\hat{\theta}_t + \beta_t V_t^{-\frac{1}{2}} \eta_t$ in \eqref{eqn:auxdd_4} and \eqref{eqn:auxdd_5} to obtain the following,
\begin{align}
\tilde{x}^T \hat{\theta}_t - \beta_t \lVert \tilde{x} \rVert_{V_t^{-1}} &\geq C_1 \label{eqn:ntaux1_sup} \\
\tilde{x}^T \hat{\theta}_t + \beta_t \lVert \tilde{x} \rVert_{V_t^{-1}} &\leq C_2 \label{eqn:ntaux2_sup}. 
\end{align}
Note that \eqref{eqn:ntaux1_sup} and \eqref{eqn:ntaux2_sup} imply that $\tilde{x} \in \hat{\mathcal{X}}^s_t$ for all $\tilde{x} \in X^*(\tilde{\theta}_t)$ from \eqref{eqn:proxysafealt}. That is, by satisfying the distributional property in \eqref{eqn:ntdef1_sup}, we can assure that $\mathbb{P} \{ \tilde{\theta}_t \in \tilde{\Theta}^{opt}_t | D_t, E_t\} > p$, and $\tilde{\theta}_t \in \tilde{\Theta}^{opt}_t$ implies that $x_t^T \tilde{\theta}_t=K$ when $D_t$ and $E_t$ hold. Finally, since the event $E_t$ implies that $\tilde{\theta}_t \in \mathcal{E}_t$, we have $\mathbb{P} \{ \tilde{\theta}_t \in \tilde{\Theta}^{opt}_t \cap \mathcal{E}_t | D_t, E_t\} > p$. 

\small
\bibliography{references}
\bibliographystyle{abbrvnat}

\end{document}